\documentclass[10pt]{article} 

\usepackage[preprint]{rlj} 

\usepackage{amssymb}            
\usepackage{amsthm}             
\usepackage{mathtools}          
\usepackage{mathrsfs}           
\usepackage{graphicx}           
\usepackage{subcaption}         
\usepackage[space]{grffile}     
\usepackage{url}                
\usepackage{tabularx}           
\usepackage{float}              
\usepackage{caption}            
\usepackage{tikz}               
\usetikzlibrary{shapes.geometric, arrows.meta, positioning, decorations.markings, backgrounds, fit}
\usepackage{algorithm}          
\usepackage{algpseudocode}      
\usepackage{adjustbox}          
\usepackage{threeparttable}     

\usepackage[normalem]{ulem}  

\newtheorem{theorem}{Theorem}
\newtheorem{lemma}{Lemma}
\newtheorem{definition}{Definition}

\newcommand{\added}[1]{\textcolor{blue}{#1}}
\newcommand{\deleted}[1]{\textcolor{red}{\sout{#1}}}

\title{A Framework for Designing Reward Functions: \\From Objectives to Features to Human-Aligned Reward Functions}

\setrunningtitle{A Framework for Designing Reward Functions}

\author{Di Yang Shi\textsuperscript{1}, W. Bradley Knox\textsuperscript{1}}
\par 

\emails{\{dyshi, bradknox\}@cs.utexas.edu}

\affiliations{
$^{1}$\textbf{University of Texas at Austin}\\
}

\contribution{
    The first reward-design framework that formalizes all three steps from outcome
    variables to a fitted linear reward: (i) a polynomial-time exact algorithm for
    causal reward term selection via reduction to min-cut; (ii) a preference
    elicitation procedure that maintains a provably conflict-free feasible weight
    region by construction, inheriting $O(n \log \kappa)$ query bounds from
    separation oracle methods; and (iii) a guided workflow connecting
    natural-language task descriptions to measurable outcome variables. Together
    these address redundancy, reward hacking, and preference misalignment in a
    unified pipeline.
    }
    {
    The cost of outcome variables and the causal relationships between them are
    assumed to be determinable, and the quality of the resulting reward function
    depends on the completeness of the initial objective elicitation which remains
    a manual process. Optimality of the weight-fitting procedure assumes a
    preference oracle that is consistent with a linear utility; in practice, human
    judgments may be noisy, unreliable, or unavailable altogether.
    }

\summary{We present a formal process to enable non-experts to instantiate and iterate on human-aligned reward functions, i.e. reward functions that adhere to a given preference ordering over trajectories. Given a task described in natural language, our process produces a linear reward function in three steps: distill the task's objectives into a set of fundamental objectives and derive measurable outcome variables that capture those fundamental objectives, select a causally representative subset of outcome variables as the reward terms, and fit weights to those reward terms via preference elicitation. Our contributions formalize the last two steps. The first is a guided workflow for deriving outcome variables. The second is a reduction of reward term selection to minimum-cost partial cover on a causal DAG, solved in polynomial time via max-flow. The third is a geometric framing of weight fitting as a convex feasibility problem iteratively narrowed by preference queries, solved by existing separation oracle methods. To the best of our knowledge, this is the first reward-design method that maintains a deterministically conflict-free feasible weight region, narrowed to a desired tolerance via a separation oracle with $O(n \log \kappa)$ preference queries.
}

\begin{document}

\maketitle  

\begin{abstract}
We present a formal process to enable non-experts to instantiate and iterate on human-aligned reward functions, i.e. reward functions that adhere to a given preference ordering over trajectories. Given a task described in natural language, our process produces a linear reward function in three steps: distill the task's objectives into a set of fundamental objectives and derive measurable outcome variables that capture those fundamental objectives, select a causally representative subset of outcome variables as the reward terms, and fit weights to those reward terms via preference elicitation. Our contributions describe the first step and formalize the latter two steps. The first is a guided workflow for deriving outcome variables. The second is a reduction of reward term selection to minimum-cost partial cover on a causal DAG, solved in polynomial time via max-flow. The third is a geometric framing of weight fitting as a convex feasibility problem iteratively narrowed by preference queries, solved by existing separation oracle methods. To the best of our knowledge, this is the first reward-design method that maintains a deterministically conflict-free feasible weight region, narrowed to a desired tolerance via a separation oracle with $O(n \log \kappa)$ preference queries.
\end{abstract}

\section{Introduction}
In reinforcement learning (RL) the design of a human-aligned reward function---one that faithfully reflects intended preferences over trajectories---is fundamental to obtaining policies that behave as desired, yet remains a process largely restricted to experts. 
Even with expert RL practitioners, reward functions are commonly derived through trial-and-error processes \citep{knox2023reward, booth2023perils} and still fall victim to common pitfalls such as reward hacking, redundancy in reward terms, and mismatched preferences in behavior from those of humans \citep{knox2023reward}.

We propose a framework (Figure~\ref{fig:reward-pipeline}) that takes as input a task described in natural language and produces a human-aligned linear reward function. The framework consists of three steps: (i) distill the task description into a set of fundamental objectives and derive measurable outcome variables that capture those fundamental objectives; (ii) select a causally representative, low-cost subset of outcome variables as the reward terms; and (iii) fit weights to those reward terms via preference elicitation.

Our contributions are the specification of these three steps in detail, including formalization of the last two steps. The first contribution (Section 3.1) is a guided workflow for distilling fundamental objectives from an initial task description and deriving measurable outcome variables from them. The second (Section 3.2) assumes that causal relationships between outcome variables are determinable and formalizes the selection of a low-cost, causally representative subset of these variables as a minimum-cost partial cover problem on a DAG, which we solve via reduction to max-flow. This turns a heuristically driven manual process with deduplication inefficiencies to a more formal, algorithmic approach that is optimal by design. The third contribution (Section~\ref{sec:fitting}) fits weights to the selected reward terms. \citet{dulacarnold2019challengesrealworldreinforcementlearning} consider multi-dimensional rewards to be one of the primary challenges in RL, particularly in how to balance tradeoffs between reward terms representing different objectives. We address this problem with an algorithmic approach that precisely queries how these tradeoffs should be made. Reframing the weight determination as a convex optimization problem in narrowing down a feasible weight space, we leverage an existing optimization method to analytically determine the best trajectory pairs to query. Such pairs correspond to a hyperplane that cuts the feasible region in half. Accordingly, our reframed approach inherits asymptotically optimal preference querying efficiency and likely-optimal runtime \citep{jiang_volumetric_bound}.

\begin{figure}[H]
    \centering
    \makebox[\textwidth][c]{\includegraphics[width=1.03\textwidth]{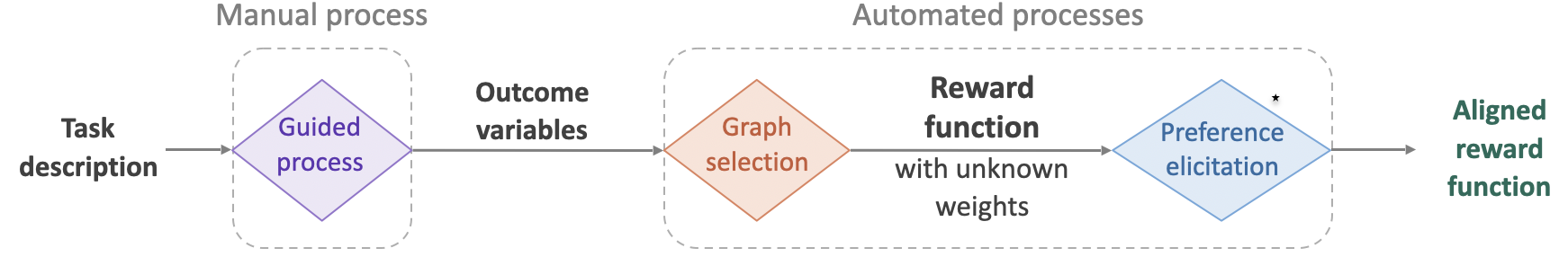}}
    \caption{An outline of our process with the colored diamonds as our contributions (sections 3.1, 3.2, and 3.3 respectively). The ``preference elicitation'' node has the caveat that it is dependent on an oracle which may be a human expert.}
    \label{fig:reward-pipeline}
\end{figure}

\section{Related Work}

We organize related work by the step of our framework that it addresses.

{\textbf{Objective elicitation and outcome variable derivation.}} Our approach to selecting fundamental objectives, a rather subjective and contextual problem, is inspired by \citet{keeney1992}'s work which advocates for first focusing on the values we strive for rather than the alternatives which are available to us and then exploring options based upon those defined values.

{\textbf{Reward term selection and representation.}} A common approach to alignment is reward shaping, the inclusion of additional reward terms towards shaping intermediary learning of policy behavior \citep{ng1999piu, wiewiora2003principled}. However, these approaches tend to fall prey to \citet{sutton2019bitter}'s bitter lesson, leaking biases towards how humans think a task should be done rather than motivating the outcomes we intend. Recent empirical studies have highlighted these shortcomings in designing reward functions even for experienced RL practitioners \citep{knox2023reward,booth2023perils}.

{\textbf{Weight fitting via preference elicitation.}} Reinforcement learning from human feedback (RLHF) has become popularized \citep{christiano2023deepreinforcementlearninghuman, ouyang2022traininglanguagemodelsfollow} but the querying process is typically Bayesian or unspecified altogether. In contrast, our querying method is deterministic with concrete bounds. A paper under review by \citet{stephane2026} explores interactive LLM-facilitated reward design with a human preference annotator; we hope to deploy our method with a similar medium but focus solely on our method's process and theory in this work. Our work is perhaps closest to \citet{sadigh2017active}'s setting in active preference learning in that it uses active preference querying using synthesized trajectories and predicates efficiency on volume removal of the feasible space. However, instead of assuming a feature function we offer a process to build one that's designed towards addressing aforementioned issues such as redundancy in reward terms and mismatch in preferences between the policy and what's intended. We also employ analytic approaches to calculate optimal cutting planes rather than sampling multiple cuts for empirical estimates of volume removal over a distribution. This further allows us to output definitive weights rather than probabilistic ones.

\section{Method}

This section presents our three contributions, which correspond to the labeled steps in Figure~\ref{fig:reward-pipeline}. Section 3.1 derives a set of candidate outcome variables from a task description. Section 3.2 selects a causally representative, low-cost subset of those candidates as the reward terms. Section~\ref{sec:fitting} fits weights to the resulting linear reward representation, producing a polytope of weight vectors consistent with elicited preferences.

\subsection{Selecting Outcome Variables from Task Objectives}

This section derives a set of candidate outcome variables from a task description in two steps. Section~\ref{sec:objectives} distills the task's initial objectives into a set of fundamental objectives. Section 3.1.2 then derives measurable outcome variables that serve as proxies for those fundamental objectives. The resulting candidate set is filtered in Section 3.2.

\subsubsection{Distilling Task Objectives Into Fundamental Objectives}
\label{sec:objectives}
This step takes a natural-language description of the task as input and produces a set of fundamental objectives via an iterative procedure. For a given task, we start by exhaustively listing objectives of the task. Then, we scrutinize each objective in this list asking why each one is important. The answers may elucidate more objectives, progressing towards more fundamental ones with each iteration. We repeat this process of iteratively conceptualizing task objectives and distilling them into more fundamental ones until the fundamental objectives satisfactorily cover the task objectives. It may be the case that an initial task objective is already fundamental and can be left as-is. We provide some examples in Table~\ref{tab:fund_to_init}.

\begin{table}[H]
\centering
\begin{tabularx}{\textwidth}{|X|X|}
\hline
\textbf{Initial Objectives} & \textbf{Fundamental Objectives} \\
\hline \hline
avoid traffic, avoid construction & minimize time \\
\hline
avoid tolls, take fuel efficient routes & minimize cost \\
\hline
drive safely, drive smoothly & minimize collisions, minimize passenger discomfort \\
\hline
\end{tabularx}
\caption{Example fundamental objectives derived from initial objectives}
\label{tab:fund_to_init}
\end{table}

\subsubsection{Deriving Measurable Outcome Variables from Objectives}
Fundamental objectives such as ``minimize passenger discomfort'' or ``maximize happiness'' are typically not directly measurable, so we cannot reward the agent for them as stated. We therefore introduce measurable proxies that we call \emph{outcome variables}. In this subsection we produce a candidate set of outcome variables, with one or more per fundamental objective.

We say that a set of outcome variables \textit{captures} a fundamental objective if there is some operation on the output of their sampling which sufficiently approximates our alignment to a fundamental objective. For each fundamental objective, we assume that there is some proxy set of measurable outcome variables that capture it; if this is not the case, one may consider redefining the fundamental objectives. We show how the previous fundamental objectives might translate to outcome variables in Table~\ref{tab:fund_to_outcome} and describe some possible issues and potential remedies in Table~\ref{tab:shortcomings}.

\begin{table}[H]
\centering
\begin{tabularx}{\textwidth}{|X|X|}
\hline
\textbf{Fundamental Objectives} & \textbf{Outcome Variables} \\
\hline \hline
minimize time & trip duration(seconds) \\
\hline
minimize cost & trip cost(\$)\\
\hline
minimize collisions & peak acceleration / jerk ($m/s^2$ or $m/s^3$) \\
\hline
minimize passenger discomfort & passenger satisfaction($\text{scalar value}$) \\
\hline
\end{tabularx}
\caption{Converting fundamental objectives to measurable outcome variables}
\label{tab:fund_to_outcome}
\end{table}

\begin{table}[H]
\centering
\begin{tabular}{p{0.3\textwidth} p{0.6\textwidth}}
\hline
\textbf{Potential Issues} & \textbf{Remedy} \\
\hline
No measurable proxy exists & Decompose the objective into more concrete sub-objectives for which proxies exist. \\[3pt]
Proxies not observable during training & Substitute with causally upstream variables observable at training time. \\[3pt]
Proxies are gameable & Diversify the proxy set so that simultaneously
gaming all proxies requires behavior that approximates the fundamental
objective. \\[3pt]
\hline
\end{tabular}
\caption{Shortcomings and remedies in identifying measurable proxies for fundamental objectives.}
\label{tab:shortcomings}
\end{table}

\subsection{A Causally-Aware Outcome Variable Selection Framework}

The candidate set produced by Section 3.1 typically contains more outcome variables than we want to include in the reward function. Outcome variables have some cost (selection, measurement, quality, etc.), and many candidates are redundant under causality: a causally downstream variable reflects the effects of its upstream variables, so measuring the upstream variable can substitute for measuring the downstream ones. This subsection formalizes the resulting selection problem on the causal DAG of outcome variables, presents an exact algorithm via reduction to max-flow, and gives a locally greedy alternative for cases where global cost or causal information is incomplete.

We propose the selection process of outcome variables as a graph algorithm. We denote a \textit{causal relationship} between two nodes $a,b$ as $a\rightarrow b$ if $a$ has a causal effect on $b$. We assume causal relationships are determinable; depending on the domain, these causal relationships may be provided by human experts or possibly even automated. Let $O$ be a set of outcome variables representing the fundamental objectives we derived and $C$ be the set of causal relationships between them. In general, causality cannot be cyclical but this fact does not preclude feedback cycles in causality through a temporal dimension. Thus, we assert that $C$ is \textit{temporally-fixed}, that is, for any subset of $C$ which forms a cycle which necessarily transcends time we effectively cut that cycle and fix a temporal segment to consider. In practice, these chicken-and-egg decisions depend on what outcome variables are first measurable in the task at hand. For example, both minimizing trip time and minimizing collisions causally influence each other; when a cycle like this arises, we fix a temporal ordering based on what is first measurable (here, trip time), eliminating the reverse edge.

\subsubsection{Problem Definition}
We formalize the selection problem as a minimum-cost cover on the causal DAG of outcome variables, where the demand nodes are the candidates we wish to account for in the reward.

Let $V=O$, $E=C$, and $G = (V, E)$ be a directed acyclic graph with node costs $c : V \to \mathbb{Q}_{\geq 0}$, $D \subseteq V$ be a set of \emph{demand nodes}, and $\mathrm{src}(G)$ denote the set of source nodes of $G$ (i.e., nodes with in-degree zero). A set $S \subseteq V$ is a
\emph{valid cover} of $D$ if every node in $D$ is covered, where coverage is defined
recursively as follows.

\begin{definition}[Coverage]
Given a selection $S \subseteq V$, a node $v \in V$ is \emph{covered} by $S$ if either
$v \in S$, or $v \notin \mathrm{src}(G)$ and every parent of $v$ is covered by $S$.
\end{definition}

The \textbf{Minimum Cost Partial Cover} (MCPC) problem asks for a valid cover $S^*$ minimizing
$\sum_{v \in S^*} c(v)$. We show that the MCPC problem can be reduced to the max-flow min-cut \citep{ford1956maxflow} problem in \hyperref[sec:appendix_reduction_of_minimum_cost_partial_cover]{\textbf{Appendix: Reduction of Minimum Cost Partial Cover}}.

\subsubsection{Reduction to Minimum Cut}

Reducing MCPC to a standard $s$-$t$ flow problem yields polynomial-time solvability via well-known max-flow algorithms. By Theorem~\ref{thm:equiv}, the MCPC equals the minimum weight vertex cut separating all sources from all demand nodes. We reduce this to a standard minimum $s$-$t$ cut problem via node splitting. We detail the construction of the MCPC problem as a flow network below.

\begin{definition}[Minimum Cost Partial Cover Flow Network]
\label{def:flow}
Given $G = (V, E)$, costs $c$, and demand set $D$, construct a flow network
$\mathcal{N} = (V', E', \sigma, \tau)$ as follows.
\begin{itemize}
    \item For each node $v \in V$: add nodes $v^{in}$ and $v^{out}$ with a directed
          edge $v^{in} \to v^{out}$ of capacity $c(v)$.
    \item For each edge $(u, v) \in E$: add a directed edge $u^{out} \to v^{in}$
          of capacity $\infty$.
    \item Add a supersource $\sigma$ and for each source $r \in \mathrm{src}(G)$:
          add edge $\sigma \to r^{in}$ of capacity $\infty$.
    \item Add a supersink $\tau$ and for each demand node $d \in D$:
          add edge $d^{out} \to \tau$ of capacity $\infty$.
\end{itemize}
\end{definition}

\subsubsection{Algorithm}

We now present \hyperref[alg:mincost]{Algorithm~\ref*{alg:mincost}} for selecting a causal cover of outcome variables from a DAG of their causal relationships.

\begin{algorithm}[H]
\caption{Minimum Cost Partial Cover As Max Flow Min Cut}
\label{alg:mincost}
\begin{algorithmic}[1]
\Require DAG $G = (V, E)$, costs $c : V \to \mathbb{Q}_{\geq 0}$, demand set $D \subseteq V$
\Ensure Minimum cost set $S^* \subseteq V$ covering all nodes in $D$
\State Construct flow network $\mathcal{N}$ as in Definition~\ref{def:flow}
\State Compute maximum $\sigma$-$\tau$ flow $f^*$ in $\mathcal{N}$ using a max-flow algorithm
\State Let $R$ be the set of nodes reachable from $\sigma$ in the residual graph of $f^*$
\State \Return $S^* \gets \{v \in V : v^{in} \in R \text{ and } v^{out} \notin R\}$
\end{algorithmic}
\end{algorithm}

\subsubsection{Runtime Complexity}
The runtime is bound by the max-flow procedure. Two well-known methods are Edmonds-Karp \citep{edmondskarp1972}, with complexity $O(|V'||E'|^2)$, and Dinic's algorithm \citep{dinic1970}, with complexity $O(|V'|^2|E'|)$, where $|V'| = 2|V| + 2$ and $|E'| = |V| + |E| + |\mathrm{src}(G)| + |D|$.

\subsubsection{Locally Greedy Approach}
In practice, the global structure assumed by the exact algorithm may not be available; the cost $c(v)$ of a node may not reduce to a single comparable metric or may be unknown, the causal relationships between some pairs may be unclear, or a practitioner may already have a proposed cover $S$ in mind. In such cases, rather than solving for the globally optimal $S^*$, one may instead verify and refine $S$ relative to a cost budget $B$.
\begin{enumerate}
    \item Select a subset $S$ of demand nodes that best represents the intended learning objectives. One may begin with $S=D$ if no expert selection is available.

    \item If $\sum_{v \in S} c(v) \leq B$, the proposed cover is within budget and we are done. Otherwise, for each $s \in S$, traverse upward through the ancestors of $s$ in $G$, replacing costly nodes in $S$ with lower-cost ancestors where possible. We note this is a heuristic: the ancestor sets of distinct nodes in $S$ may overlap, causing costs to be double-counted relative to the optimal $S^*$.

    \item If $S$ does not provide sufficient learning signal after refinement, additional reward shaping terms may be introduced from higher in $G$, subject to the remaining budget $B - \sum_{v \in S} c(v)$.
\end{enumerate}

\subsection{Fitting the Reward Function's Parameters to Preferences}
\label{sec:fitting}

Section 3.2 produced a set of reward terms. To complete the reward function we need to (a) decide how those terms combine into a scalar reward and (b) fit the parameters of that combination. We assume a linear combination and justify this briefly below before describing the weight-fitting procedure. In their seminal work on inverse reinforcement learning, \citet{ng_irl} argued that linear reward functions are necessary for computational feasibility, are generally sufficient, cover a representative subset of states, and support constraint relaxation with penalty terms---with nonlinearities captured in the basis functions instead. Linear reward functions also offer interpretability and enable core techniques such as successor features \citep{barreto_successor}; critically, our approach relies on convex optimization concepts (half-spaces, analytic centers, feasible regions) that are only sensible in a linear setting.

We then fit the weights by iteratively shrinking a feasible region in weight space using preference queries. Section 3.3.1 sets up the geometry of this feasible region, and Section 3.3.2 presents the synthetic-trajectory algorithm that proposes preference queries, along with its query complexity and runtime. The output is a polytope of weight vectors consistent with all elicited preferences.

\subsubsection{Geometric Problem Description}
We now describe the geometry that allows us to recast weight fitting as a convex feasibility problem solvable by separation oracle methods.

Let $\textsc{Propose}$ be the function that proposes trajectory pairs for preference evaluation, and $\textsc{Decide}$ be the preference decision function (a human expert, algorithm, oracle, etc.). Given outcome variables selected as reward terms that capture the entire trajectory $x_1,\ldots,x_n$, our goal is to recover the weights $w_1,\ldots,w_n$ that best approximate $\textsc{Decide}$, defining a target reward function $R^*(x) = w^{*\top} x \in \mathbb{R}$, where $w^* = (w_1^*, \ldots, w_n^*)$ is the true weight vector and $x = (x_1, \ldots, x_n)$ is the feature vector for a trajectory. Because the reward is scale-invariant, only the direction of the weight vector matters; we identify the weight vector with a point on $S^{n-1}=\{w \in \mathbb{R}^n : \|w\| = 1\}$, reducing the problem to $n-1$ degrees of freedom.

Each call $\textsc{Propose}(\mathcal{C})$ yields an inequality constraint corresponding to a hyperplane through the origin. Concretely, \textsc{Propose} selects a trajectory pair $(\tau^+, \tau^-)$ with feature vectors $x(\tau^+), x(\tau^-) \in \mathbb{R}^n$; the induced query normal is $q = \Delta x / \|\Delta x\|$ where $\Delta x = x(\tau^+) - x(\tau^-)$. Once \textsc{Decide} returns sign $s \in \{-1,+1\}$, the resulting half-space constraint on the weight vector is $\{w : s\, q^\top w \geq 0\}$. Geometrically, we iteratively refine a convex polyhedral cone, projected onto $S^{n-1}$, to localise $w^*$. Let $\mathcal{C} \subset \mathcal{H}^n$ denote the set of accumulated constraints where $\mathcal{H}^n$ is the set of n-dimensional half-spaces, and $\kappa = nR/\epsilon$ where $R$ is the radius of the initial containing box and $\epsilon$ is the desired error tolerance.

\subsubsection{Synthetic Trajectory Algorithms}

Our primary approach narrows $\mathcal{C}$ via $\mathcal{C} \leftarrow \mathcal{C} \cup \{\textsc{Decide}(\textsc{Propose}(\mathcal{C}))\}$. Trajectory pairs submitted for preference queries are constructed entirely from scratch rather than sampled from a fixed dataset, enabling us to pose arbitrary hypothetical comparisons---analogous to asking ``would you rather'' questions when eliciting another person's values.

This framework is equivalent to the \emph{separation oracle} (SO) setting of \citet{grotschel_so}. The most prominent SO method is the analytic center cutting plane method (ACCPM)~\citep{atkinson_accpm}; the most asymptotically efficient is the volumetric center method~\citep{vaidya_volumetric} (with computational improvements by \citealt{jiang_volumetric_bound}). Both use $O(n \log \kappa)$ oracle calls and differ primarily in how they select the query point: ACCPM uses a distance-based criterion (the analytic center), while the volumetric center uses a volume-based criterion. As an implementation note, we substitute the positive unit hypercube as the initial feasible region and normalize output weights as needed; this implicitly restricts $w \geq 0$, which assumes outcome variables are oriented so that larger values are preferred---if negative weights are required, a standard hypercube $[-1,1]^n$ should be used instead.

In practice, we expect $n$ to be small, so ACCPM is sufficient in most cases; when oracle calls are prohibitively expensive, the volumetric center method is preferred as it achieves greater volume reduction per call~\citep{vaidya_volumetric}. Once the query point is determined by the center computation (see Table~\ref{tab:accpm}), constructing the half-space constraint $\{w : s\,q^\top w \geq 0\}$ and synthesizing a corresponding trajectory pair $(\tau^+, \tau^-)$ with $x(\tau^+) - x(\tau^-) = \lambda q$, where $\lambda > 0$ is chosen to produce outcome magnitudes that are meaningful to the evaluator, requires $O(n)$ additional work. Table~\ref{tab:accpm} summarizes the two methods.

\begin{table}[H]
\centering
\footnotesize
\begin{tabular}{p{0.22\textwidth} c p{0.3\textwidth} p{0.22\textwidth}}
\hline
\textbf{Method} & \textbf{\# Oracle Calls} & \textbf{Runtime} & \textbf{Barrier Function} \\
\hline
Analytic Center \citep{atkinson_accpm} &
$O(n \log\kappa)$ &
$O\!\left(n^{\omega+1} \log^2\kappa + (n\log\kappa)^{2+\omega/2}\right)$ &
$-\sum_{i=1}^{m} \log(b_i - a_i^T x)$ \\[4pt]
Volumetric Center \citep{vaidya_volumetric} &
$O(n\log\kappa)$ &
$O(n^3\log\kappa)$ \citep{jiang_volumetric_bound} &
$-\log \det \sum_{i=1}^{m} \frac{a_i a_i^T}{(b_i - a_i^T x)^2}$ \\
\hline
\end{tabular}
\vspace{2pt}
{\footnotesize $n$: dimension;\quad
$\kappa = nR/\epsilon$;\quad
$\epsilon$: error tolerance;\quad
$R$: radius of initial containing box;\quad
$\omega$: matrix multiplication exponent;\quad
$m = |\mathcal{C}|$}
\caption{Comparison of ACCPM~\citep{atkinson_accpm} and the volumetric center method~\citep{vaidya_volumetric,jiang_volumetric_bound}.}
\label{tab:accpm}
\end{table}

We summarize some key benefits that synthetic trajectory generation afford our method.
\begin{itemize}
    \item \textbf{Consistency.} Because SO methods find a point within the feasible region at each step, the resulting constraint set is guaranteed to be consistent and conflict-free. This avoids distributional assumptions and conflict-resolution procedures that arise when preferences are elicited over a fixed trajectory dataset. Essentially, we never ``ask the same (preference) question twice''.

    \item \textbf{Complexity guarantees.} The convergence bounds of SO methods are predicated on the volume removed from the feasible polyhedron at each step---a quantity we can guarantee when generating trajectories synthetically.

    \item \textbf{Controllability.} Trajectories for preference elicitation can be modified before presentation. For example, numerical outcomes can be rounded to improve human readability, or trajectories can be shifted toward a target distribution to make comparisons easier for non-expert evaluators who may struggle with out-of-distribution hypotheticals.
\end{itemize}

\section{Conclusion}

We have presented a framework that takes a task description and produces a human-aligned linear reward function, along with formalizations and algorithms for three contribution steps: a guided workflow for deriving outcome variables, a reduction of reward term selection to minimum-cost partial cover (solved via max-flow), and a geometric framing of weight fitting as a convex feasibility problem narrowed by preference queries. These steps address three recurring failure modes in reward design: redundancy (mitigated by causal cover), reward hacking (mitigated by grounding reward terms in elicited objectives rather than shaped intermediate behaviors), and preference misalignment (mitigated by producing a consistent feasible weight region by construction). We hope this framework serves as a principled foundation for further work on automated reward design and can be a step towards an interactive system that enables non-RL practitioners to design and iterate on reward functions.

\newpage

\bibliography{refs}
\bibliographystyle{rlj}

\newpage

\appendix

\section{Reduction of Minimum Cost Partial Cover}
\label{sec:appendix_reduction_of_minimum_cost_partial_cover}

\subsection{Preliminaries: Flow Networks \citep{ford1956maxflow}}

A \emph{flow network} is a directed graph $\mathcal{N} = (V', E', \sigma, \tau)$ where
each edge $e \in E'$ has a non-negative \emph{capacity} $\mathrm{cap}(e) \in \mathbb{Q}_{\geq 0} \cup \{\infty\}$,
and two distinguished nodes play special roles.

\begin{definition}[Supersource and Supersink]
The \emph{supersource} $\sigma \in V'$ is a node with no incoming edges, representing
the origin of all flow. The \emph{supersink} $\tau \in V'$ is a node with no outgoing
edges, representing the destination of all flow.
\end{definition}

\begin{definition}[Flow]
A \emph{flow} $f : E' \to \mathbb{R}_{\geq 0}$ is an assignment of values to edges
satisfying: (i) \emph{capacity constraints}: $f(e) \leq \mathrm{cap}(e)$ for all $e \in E'$,
and (ii) \emph{conservation}: for every node $v \notin \{\sigma, \tau\}$, the total
flow into $v$ equals the total flow out of $v$. The \emph{value} of a flow is the
total flow out of $\sigma$.
\end{definition}

\begin{definition}[$s$-$t$ Cut]
A \emph{$\sigma$-$\tau$ cut} is a partition $(A, B)$ of $V'$ with $\sigma \in A$ and
$\tau \in B$. Its \emph{capacity} is the total capacity of edges from $A$ to $B$:
\[
\mathrm{cap}(A, B) = \sum_{\substack{e = (u,v) \\ u \in A,\, v \in B}} \mathrm{cap}(e).
\]
A \emph{minimum cut} is a $\sigma$-$\tau$ cut of minimum capacity.
\end{definition}

\begin{definition}[Residual Graph]
Given a flow $f$ in $\mathcal{N}$, the \emph{residual graph} $\mathcal{N}_f$ has the
same node set $V'$ and contains, for each edge $e = (u,v) \in E'$: a
\emph{forward edge} $(u,v)$ with residual capacity $\mathrm{cap}(e) - f(e)$ if
$f(e) < \mathrm{cap}(e)$, and a \emph{backward edge} $(v,u)$ with residual capacity
$f(e)$ if $f(e) > 0$.
\end{definition}

\begin{theorem}[Max-Flow Min-Cut]
In any flow network, the maximum value of a $\sigma$-$\tau$ flow equals the minimum
capacity of a $\sigma$-$\tau$ cut. Furthermore, a flow $f^*$ is maximum if and only
if $\tau$ is not reachable from $\sigma$ in the residual graph $\mathcal{N}_{f^*}$.
In this case, setting $A$ to be the nodes reachable from $\sigma$ in $\mathcal{N}_{f^*}$
and $B = V' \setminus A$ yields a minimum cut $(A, B)$.
\end{theorem}

\subsection{Structural Characterization}

We first show that valid covers are exactly the sets that hit every source-to-demand path.
Let $\mathrm{src}(G)$ denote the set of source nodes of $G$ (nodes with in-degree zero).

\begin{lemma}
\label{lem:equiv}
A set $S \subseteq V$ is a valid cover of $D$ if and only if $S$ intersects every
directed path from any source $r \in \mathrm{src}(G)$ to any demand node $d \in D$.
\end{lemma}

\begin{proof}
$(\Rightarrow)$ Suppose $S$ is a valid cover and let $P = v_0 \to v_1 \to \cdots \to v_k$
be a directed path with $v_0 \in \mathrm{src}(G)$ and $v_k \in D$. Suppose for
contradiction that $P \cap S = \emptyset$. Since $v_0$ is a source it has no parents,
so it is covered only if $v_0 \in S$ --- a contradiction. Hence $S$ must hit every
such path.

$(\Leftarrow)$ Suppose $S$ hits every source-to-demand path. We show every $d \in D$
is covered by induction on the length of the longest source-to-$d$ path. For the base
case, if $d$ has no parents then $d \in \mathrm{src}(G)$, and the path of length zero
from $d$ to itself must be hit by $S$, so $d \in S$ and $d$ is covered.

For the inductive step, if $d \notin S$ then consider any parent $u$ of $d$. Every
directed path from a source to $u$ can be extended by the edge $u \to d$ to a
source-to-$d$ path, which $S$ must hit. Since $S$ does not hit $u \to d$ at $d$
(as $d \notin S$), it must hit the path at or before $u$. Hence every source-to-$u$
path is hit by $S$. By the inductive hypothesis $u$ is covered. Since $u$ was
arbitrary, every parent of $d$ is covered, so $d$ is covered.
\end{proof}

\subsection{Reduction to Min-Cut Proof}

\begin{theorem}
\label{thm:equiv}
The minimum cost partial cover of $D$ equals the minimum $\sigma$-$\tau$ cut in
$\mathcal{N}$, and hence equals the maximum $\sigma$-$\tau$ flow in $\mathcal{N}$
by the max-flow min-cut theorem.
\end{theorem}

\begin{proof}
Every finite-capacity edge in $\mathcal{N}$ has the form $v^{in} \to v^{out}$ for
some $v \in V$, with capacity $c(v)$. A $\sigma$-$\tau$ cut of finite capacity must
consist entirely of such edges. Let $S = \{v \in V : (v^{in} \to v^{out}) \in \mathrm{cut}\}$
be the corresponding set of nodes.

Every directed path from $\sigma$ to $\tau$ in $\mathcal{N}$ has the form
\[
\sigma \to r^{in} \to r^{out} \to \cdots \to d^{in} \to d^{out} \to \tau
\]
for some $r \in \mathrm{src}(G)$ and $d \in D$, and corresponds exactly to a
source-to-demand path in $G$. The cut severs all such paths iff $S$ hits every
source-to-demand path in $G$. By Lemma~\ref{lem:equiv} this is exactly the condition
for $S$ to be a valid cover. The cut capacity equals $\sum_{v \in S} c(v)$, so the
minimum cut corresponds to the minimum cost valid cover.
\end{proof}

\end{document}